\documentclass[10pt,twocolumn,letterpaper]{article}

\usepackage{cvpr}
\usepackage{times}
\usepackage{epsfig}
\usepackage{graphicx}
\usepackage{amsmath}
\usepackage{amssymb}

\usepackage{graphicx}
\usepackage{epsfig}
\usepackage{graphicx}
\usepackage{amsmath}
\usepackage{amssymb}
\usepackage{amsthm}
\usepackage{multirow}
\usepackage{algorithmic}
\usepackage{algorithm}
\usepackage{latexsym}
\usepackage{indentfirst}
\usepackage{bbm}
\usepackage{bm}
\usepackage{subfigure}
\usepackage{enumitem}

\graphicspath{{illustrations/}}

\usepackage[pagebackref=true,breaklinks=true,letterpaper=true,colorlinks,bookmarks=false]{hyperref}

\cvprfinalcopy 

\def\cvprPaperID{525} 

\ifcvprfinal\fi

\newcommand{\cS}{\mathcal{S}}

\newcommand{\bv}{\mathbf{v}}
\newcommand{\bx}{\mathbf{x}}

\newcommand{\bD}{\mathbf{D}}

\newcommand{\bS}{\mathbf{S}}

\newcommand{\bmX}{\bm X}

\newcommand{\bzero}{\mathbf 0}

\newcommand{\bL}{{\mathbf{L}}}
\newcommand{\bW}{{\mathbf{W}}}

\newcommand{\bY}{{\mathbf{Y}}}

\newcommand{\bZ}{{\mathbf{Z}}}
\newcommand{\balpha}{\bm{\alpha}}

\newcommand{\N}{{\rm I}\kern-0.18em{\rm N}}

\newcommand{\R}{{\rm I}\kern-0.18em{\rm R}}
\newcommand{\h}{{\rm I}\kern-0.18em{\rm H}}
\newcommand{\K}{{\rm I}\kern-0.18em{\rm K}}
\newcommand{\p}{{\rm I}\kern-0.18em{\rm P}}
\newcommand{\E}{{\rm I}\kern-0.18em{\rm E}}
\newcommand{\Z}{{\rm Z}\kern-0.18em{\rm Z}}

\newcommand{\1}{{\rm 1}\kern-0.25em{\rm I}}

\newcommand{\pn}{\p_{\kern-0.25em n}}
\newcommand{\pnm}{\p_{\kern-0.25em n,m}}
\newcommand{\psubm}{\p_{\kern-0.25em m}}

\newcommand{\BigO}[1]{{\operatorname{O}}}

\DeclareMathOperator*{\argmin}{arg\,min}

\newtheorem{MyTheorem}{Theorem}

\newtheorem{MyProposition}{Proposition}

\begin{document}

\title{Learning with $\ell^{0}$-Graph: $\ell^{0}$-Induced Sparse Subspace Clustering}

\author{}

\maketitle

\begin{abstract}
Sparse subspace clustering methods, such as Sparse Subspace Clustering (SSC) \cite{ElhamifarV13} and $\ell^{1}$-graph \cite{YanW09,ChengYYFH10}, are effective in partitioning the data that lie in a union of subspaces. Most of those methods use $\ell^{1}$-norm or $\ell^{2}$-norm with thresholding to impose the sparsity of the constructed sparse similarity graph, and certain assumptions, e.g. independence or disjointness, on the subspaces are required to obtain the subspace-sparse representation, which is the key to their success. Such assumptions are not guaranteed to hold in practice and they limit the application of sparse subspace clustering on subspaces with general location. In this paper, we propose a new sparse subspace clustering method named $\ell^{0}$-graph. In contrast to the required assumptions on subspaces for most existing sparse subspace clustering methods, it is proved that subspace-sparse representation can be obtained by $\ell^{0}$-graph for arbitrary distinct underlying subspaces almost surely under the mild i.i.d. assumption on the data generation. We develop a proximal method to obtain the sub-optimal solution to the optimization problem of $\ell^{0}$-graph with proved guarantee of convergence. Moreover, we propose a regularized $\ell^{0}$-graph that encourages nearby data to have similar neighbors so that the similarity graph is more aligned within each cluster and the graph connectivity issue is alleviated. Extensive experimental results on various data sets demonstrate the superiority of $\ell^{0}$-graph compared to other competing clustering methods, as well as the effectiveness of regularized $\ell^{0}$-graph.
\end{abstract}
\section{Introduction}
Clustering is a common unsupervised data analysis method which partitions data into a set of self-similar clusters. High dimensionality of data often imposes difficulty on clustering. For example, model-based clustering methods, such as Gaussian Mixture Model (GMM) that models the data by a mixture of parametric distributions, suffer from the curse of dimensionality when fitting a statistical model to the data \cite{Fraley02}.


Based on the observation that high dimensional data often lie in a set of low-dimensional subspaces in many practical scenarios, subspace clustering algorithms \cite{Vidal11} aim to partition the data such that data belonging to the same subspace are identified as one cluster. Among various subspace clustering algorithms, the ones that employ sparsity prior, such as Sparse Subspace Clustering (SSC) \cite{ElhamifarV13} and $\ell^{1}$-graph \cite{YanW09,ChengYYFH10}, have been proven to be effective in separating the data in accordance with the subspaces that the data lie in under certain assumptions.

Sparse subspace clustering methods construct sparse similarity graph by sparse representation of the data, where the vertices represent the data, and an edge is between two vertices whenever one participates the spare representation of the other. Thanks to the subspace-sparse representation, the nonzero elements in the sparse representation of each datum in a subspace correspond to the data points in the same subspace, so that vertices corresponding to different subspaces are disconnected in the sparse similarity graph, leading to their compelling performance with spectral clustering \cite{Ng01} applied on such graph.

\cite{ElhamifarV13} proves that when the subspaces are independent or disjoint, then subspace-sparse representations can be obtained by solving the canonical sparse coding problem using data as the dictionary under certain conditions on the rank, or singular value of the data matrix and the principle angle between the subspaces respectively. Under the independence assumption on the subspaces, low rank representation \cite{LiuLY10,Liu12} is also proposed to recover the subspace structures. Relaxing the assumptions on the subspaces to allowing overlapping subspaces, the Greedy Subspace Clustering \cite{ParkCS14} and the Low-Rank Sparse Subspace Clustering \cite{Wang13} achieve subspace-sparse representation with high probability. However, their results rely on the semi-random model which assumes the data in each subspace are generated i.i.d. uniformly on the unit sphere in that subspace as well as certain additional conditions on the size and dimensionality of the data. In addition, the geometric analysis in \cite{Soltanolkotabi2012} also adopts the semi-random model and it handles overlapping subspaces.

To avoid the non-convex optimization problem incurred by $\ell^{0}$-norm, most of the sparse subspace clustering or sparse graph based clustering methods use $\ell^{1}$-norm \cite{YanW09,ChengYYFH10,ElhamifarV11,ElhamifarV13,YYZRl1graphBMVC2014} or $\ell^{2}$-norm with thresholding \cite{Peng2015robust} to impose the sparsity on the constructed similarity graph. In addition, $\ell^{1}$-norm has been widely used as a convex relaxation of $\ell^{0}$-norm for efficient sparse coding algorithms \cite{jenatton2010proximal,Mairal2010,MairalBPSZ08}. On the other hand, sparse representation methods such as \cite{Mancera2006} that directly optimize objective function involving $\ell^{0}$-norm demonstrate compelling performance compared to its $\ell^{1}$-norm counterpart. It remains an interesting question whether sparse subspace clustering equipped with $\ell^{0}$-norm, which is the origination of the sparsity that counts the number of nonzero elemens, has advantage in obtaining the subspace-sparse representation. In this paper, we propose $\ell^{0}$-graph which employs $\ell^{0}$-norm to enforce the sparsity of the similarity graph. This paper offers three contributions:
\begin{itemize}
  \item[] \textbf{Theoretical Results on $\ell^{0}$-Induced Almost Surely Subspace-Sparse Representation} We present the theory of the $\ell^{0}$-induced sparse subspace clustering by $\ell^{0}$-graph, which shows that $\ell^{0}$-graph renders subspace-sparse representation almost surely under minimum assumptions on the underlying subspaces the data lie in, i.e. subspaces are distinct. To the best of our knowledge, this is the mildest assumption on the subspaces compared to most existing sparse subspace clustering methods. Furthermore, our theory assumes that the data in each subspace are generated i.i.d. from arbitrary continuous distribution supported on that subspace, which is milder than the assumption of semi-random model in \cite{ParkCS14} and \cite{Wang13} that assume the data are i.i.d. uniformly distributed on the unit sphere in each subspace.
  \item[] \textbf{Efficient Optimization} The optimization problem of $\ell^{0}$-graph is NP-hard and it is impractical to pursue the global optimal solution. Instead, we develop an efficient proximal method to obtain a sub-optimal solution with convergence guarantee.
  \item[] \textbf{Regularized $\ell^{0}$-Graph} In order to obtain a sparse similarity graph where neighboring data have similar neighbors so as to encourage the graph connectivity within each cluster, we propose Regularized $\ell^{0}$-graph that incorporates an regularization term into the objective of $\ell^{0}$-graph. Moreover, we have implemented both $\ell^{0}$-graph and regularized $\ell^{0}$-graph in CUDA C programming language for significant speedup by parallel computing.
\end{itemize}
Note that SSC-OMP \cite{Dyer13a} adopts Orthogonal Matching Pursuit (OMP) \cite{Tropp04} to choose neighbors for each datum in the sparse similarity graph, which can be interpreted as approximately solving a $\ell^{0}$ problem. However, SSC-OMP does not present the theoretical properties of the $\ell^{0}$-induced sparse subspace clustering, and the experimental results show the significant performance advantage of $\ell^{0}$-graph over the OMP-graph. OMP-graph solves the $\ell^{0}$ problem of $\ell^{0}$-graph by OMP, so that it is equivalent to SSC-OMP for clustering. Although our optimization algorithm only obtains a sub-optimal solution to the objective of $\ell^{0}$-graph, we give theory about $\ell^{0}$-induced subspace structures and extensive experimental results show the effectiveness of our model.

The remaining parts of the paper are organized as follows. The representative subspace subspace clustering methods, SSC and $\ell^{1}$-graph, are introduced in the next subsection, and then the detailed formulation of $\ell^{0}$-graph and regularized $\ell^{0}$-graph is illustrated. We then show the clustering performance of the proposed models, and conclude the paper. We use bold letters for matrices and vectors, and regular lower letter for scalars throughout this paper. The bold letter with superscript indicates the corresponding column of a matrix, and the bold letter with subscript indicates the corresponding element of a matrix or vector. $\|\cdot\|_F$ and $\|\cdot\|_p$ denote the Frobenius norm and the $\ell^{p}$-norm, and ${\rm diag}(\cdot)$ indicates the diagonal elements of a matrix.

\subsection{Sparse Subspace Clustering and $\ell^{1}$-Graph}
Sparse coding methods represent an input signal by a linear combination of only a few atoms of a dictionary, and the sparse coefficients are named sparse code. Sparse coding has been broadly applied in machine learning and signal processing, and sparse code is extensively used as a discriminative and robust feature representation \cite{YangYGH09,ChengSRL2013,ZhangGLXA13,YYZRl1graphBMVC2014}

SSC \cite{ElhamifarV13} and $\ell^{1}$-graph \cite{YanW09,ChengYYFH10} employ sparse representation of the data to construct the sparse similarity graph. With the data ${\bm X}=[ {{\bx_1},\ldots ,{\bx_n}} ] \in {\R^{d \times n}}$ where $n$ is the size of the data and $d$ is the dimensionality, SSC and $\ell^{1}$-graph solves the following sparse coding problem:
\begin{small}\begin{align}\label{eq:ssc-l1}
\mathop {\min }\limits_{{\balpha}} {\| {{\balpha}} \|_1}\quad s.t.\;{\bmX} = {{\bmX}}{\balpha},\,\, {\rm diag}(\balpha) = \bzero
\end{align}\end{small}
Both SSC and $\ell^{1}$-graph construct a sparse similarity graph $G = ( {{\bm X},{\mathbf W}} )$ where the data ${\bmX}$ are represented as vertices, $\bW$ is the graph weight matrix of size $n \times n$ and $\bW_{ij}$ indicates the similarity between $\bx_i$ and $\bx_j$, $\bW$ is set by the sparse codes $\balpha$ as below:
\begin{small}\begin{align}\label{eq:W}
{\bW_{ij}}=({|{\balpha_{ij}}|+|{\balpha_{ji}}|})/{2} \quad 1 \le i,j \le n
\end{align}\end{small}
Furthermore, suppose the underlying subspaces that the data lie in are independent or disjoint, SSC \cite{ElhamifarV13} proves that the optimal solution to (\ref{eq:ssc-l1}) is the subspace-sparse representation under several additional conditions. \textit{The sparse representation $\balpha$ is called subspace-sparse representation if the nonzero elements of $\balpha^i$, namely the sparse representation of the datum $\bx_i$,  correspond to the data points in the same subspace as $\bx_i$}. Therefore, vertices corresponding to different subspaces are disconnected in the sparse similarity graph. With the subsequent spectral clustering \cite{Ng01} applied on such sparse similarity graph, compelling clustering performance is achieved.

Allowing some tolerance for inexact representation, the literature often turns to solve the following problem for SSC and $\ell^{1}$-graph:
\begin{small}\begin{align*}
\mathop {\min }\limits_{{\balpha}} {\| {{\balpha}} \|_1}\quad s.t.\;\|{\bmX} - {{\bmX}}{\balpha}\|_F \le \delta,\,\, {\rm diag}(\balpha) = \bzero
\end{align*}\end{small}
which is equivalent to the following problem
\begin{small}\begin{align}\label{eq:ssc-l1-lasso}
\mathop {\min }\limits_{{\balpha}} {\|\bmX - \bmX \balpha\|_F^2 + {\lambda_{\ell^{1}}}\|{\balpha}\|_1} \quad s.t. \,\, {\rm diag}(\balpha) = \bzero
\end{align}\end{small}
where ${\lambda_{\ell^{1}}}>0$ is a weighting parameter for the $\ell^{1}$ term.

\begin{table*}[ht]
\centering
\footnotesize
\caption{Assumptions on the subspaces and random data generation (for randomized part of the algorithm) for different sparse subspace clustering methods. Note that $S_1 < S_2 < S_3 < S_4$, $D_1 < D_2$, and the assumption on the right hand side of $<$ is milder than that on the left hand side.}
\begin{tabular}{|c|c|c|c|c|c|c|c|c|c|c|}
  \hline
  Assumption on Subspaces                           &Explanation            \\\hline
  $S_1$:Independent Subspaces (\cite{LiuLY10,Liu12}) &  ${\rm Dim} [\cS_1 \otimes \cS_2 \ldots \cS_K] = \sum\limits_k {\rm Dim}[\cS_k]$                        \\ \hline
  $S_2$:Disjoint Subspaces (\cite{ElhamifarV13})            & $\cS_k \cap \cS_{k'} = \bzero$ for $k \neq k'$                             \\ \hline
  $S_3$:Overlapping Subspaces  (\cite{ParkCS14,Wang13,Soltanolkotabi2012})        &  ${\rm Dim} [\cS_k \cap \cS_{k'}] < \min\{{\rm Dim}[\cS_k],{\rm Dim}[\cS_{k'}]\}$ for $k \neq k'$   \\ \hline
  $S_4$:Distinct Subspaces  ($\ell^{0}$-Graph)          &$\cS_k \neq \cS_{k'}$ for $k \neq k'$                        \\ \hline\hline
  Assumption on Random Data Generation  &Explanation \\\hline
  $D_1$:Semi-Random Model   (\cite{ParkCS14,Wang13,Soltanolkotabi2012})           &The data in each subspace are generated i.i.d. uniformly on the unit sphere in that subspace. \\ \hline
  $D_2$:IID      ($\ell^{0}$-Graph)                     &The data in each subspace are generated i.i.d. from arbitrary continuous distribution supported on that subspace. \\ \hline
\end{tabular}
\label{table:assumptions}
\end{table*}

\section{$\ell^{0}$-Induced Sparse Subspace Clustering}
In this paper, we investigate $\ell^{0}$-induced sparse subspace clustering method, which solves the following $\ell^{0}$ problem:
\begin{small}\begin{align}\label{eq:ssc-l0}
\mathop {\min }\limits_{{\balpha}} {\| {{\balpha}} \|_0}\quad s.t.\;{\bmX} = {{\bmX}}{\balpha},\,\, {\rm diag}(\balpha) = \bzero
\end{align}\end{small}
We then give the theorem about $\ell^{0}$-induced almost surely subspace-sparse representation, and the proof is presented in the supplementary document for this paper.
\begin{MyTheorem}\label{theorem::l0-ssc}
(\textit{$\ell^{0}$-Induced Almost Surely Subspace-Sparse Representation})
Suppose the data ${\bm X}=[ {{\bx_1},\ldots ,{\bx_n}} ] \in {\R^{d \times n}}$ lie in a union of $K$ distinct subspaces $\{\cS_k\}_{k=1}^K$ of dimensions $\{d_k\}_{k=1}^K$, i.e. $\cS_k \neq \cS_{k'}$ for $k \neq k'$. Let $\bmX^{(k)} \in \R^{d \times n_k}$ denotes the data that belong to subspace $\cS_k$, and $\sum\limits_{k=1}^K n_k = n$. When $n_k \ge d_k+1$, if the data belonging to each subspace are generated i.i.d. from some unknown distribution supported on that subspace, then with probability $1$, the optimal solution to (\ref{eq:ssc-l0}), denoted by $\balpha^*$, is a subspace-sparse representation, i.e. nonzero elements in ${\balpha^*}^i$ corresponds to the data that lie in the same subspace as $\bx_i$.
\end{MyTheorem}
Based on the above theorem, we propose $\ell^{0}$-graph that solves (\ref{eq:ssc-l0}) and uses the sparse representation to build the sparse similarity graph for clustering. According to Theorem~\ref{theorem::l0-ssc}, $\ell^{0}$-induced sparse subspace clustering method (\ref{eq:ssc-l0}) obtains the subspace-sparse representation almost surely under minimum assumption on the subspaces, i.e. it only requires that the subspaces be distinct. To the best of our knowledge, this is the mildest assumption on the subspaces for most existing sparse subspace clustering methods. Moreover, the only assumption on the data generation is that the data in each subspace are i.i.d. random samples from arbitrary continuous distributions supported on that subspace. In the light of assumed data distribution, such assumption on the data generation is much milder than the assumption of the semi-random model in (\cite{ParkCS14,Wang13,Soltanolkotabi2012}) (note that the data can always be normalized to have unit norm and reside on the unit sphere).  Table~\ref{table:assumptions} summarizes different assumptions on the subspaces and random data generation for different sparse subspace clustering methods. It can be seen that $\ell^{0}$-graph has mildest assumption on both subspaces and the random data generation.

\section{Optimization of $\ell^{0}$-Graph}
We introduce the optimization algorithm for $\ell^{0}$-graph in this section. Similar to the case of SSC and $\ell^{1}$-graph, by allowing tolerance for inexact representation, we turn to optimize the following $\ell^{0}$ problem
\begin{small}\begin{align}\label{eq:l0graph}
\mathop {\min }\limits_{{\balpha}} L(\balpha) = {\|\bmX - \bmX \balpha\|_F^2 + {\lambda}\|{\balpha}\|_0} \quad s.t. \,\, {\rm diag} = \bzero
\end{align}\end{small}
Problem (\ref{eq:l0graph}) is NP-hard, and it is impractical to seek for its global optimal solution. The literature extensively resorts to approximate algorithms, such as Orthogonal Matching Pursuit \cite{Tropp04}, or that uses surrogate functions \cite{Hyder09}, for $\ell^{0}$ problems. Inspired by recent advances in solving non-convex optimization problems by proximal linearized method \cite{BoltePAL2014} and the application of this method to $\ell^{0}$-norm based dictionary learning \cite{BaoJQS14}, we propose an iterative proximal method to optimize (\ref{eq:l0graph}) and obtain a sub-optimal solution with proved convergence guarantee. In the following text, the superscript with bracket indicates the iteration number of the proposed proximal method.

In $t$-th iteration of our proximal method for $t \ge 1$, gradient descent is performed on the squared loss term of (\ref{eq:l0graph}), i.e. $Q(\balpha) = \|\bmX - \bmX \balpha\|_F^2$, to obtain
\begin{small}\begin{align}\label{eq:l0graph-proximal-step1}
\tilde {\balpha}^{(t)} = {\balpha}^{(t-1)} - \frac{2}{{\tau}s} ({\bmX^\top}{\bmX}{\balpha^{(t-1)}}-{\bmX^\top}{\bmX})
\end{align}\end{small}
where $\tau$ is any constant that is greater than $1$, and $s$ is the Lipschitz constant for the gradient of function $Q(\cdot)$, namely
\begin{small}\begin{align}\label{eq:lipschitz-L}
\|\nabla Q(\bY) - \nabla Q(\bZ)\|_F \le s \|\bY-\bZ\|_F, \,\, \forall \, \bY,\bZ \in \R^{n \times n}
\end{align}\end{small}
Then ${\balpha}^{(t)}$ is the solution to the following $\ell^{0}$ regularized problem:
\begin{small}\begin{align}\label{eq:l0graph-subprob}
&{\balpha}^{(t)} = \argmin \limits_{\bv \in \R^{n \times n}} {\frac{{\tau} s}{2}\|\bv - {\tilde {\balpha}^{(t)}}\|_F^2 + {\lambda}\|\bv\|_0} \\ &\quad s.t. \,\, {\rm diag} (\bv) = \bzero \nonumber
\end{align}\end{small}
It can be verified that (\ref{eq:l0graph-subprob}) has closed-form solution, i.e.
\begin{small}\begin{align}\label{eq:l0graph-proximal-step2}
&{\balpha}_{ij}^{(t)} =
\left\{
\begin{array}
    {r@{\quad:\quad}l}
    0 & {|{\tilde {\balpha}_{ij}^{(t)}}| < \sqrt{\frac{2\lambda}{{\tau}s}} \,\, {\rm or } \,\, i = j   } \\
    {\tilde {\balpha}_{ij}^{(t)}} & {\rm otherwise}
\end{array}
\right.
\end{align}\end{small}
\noindent for $1 \le i,j \le n$. The iterations start from $t=1$ and continue until the sequence $\{L({\balpha}^{(t)})\}$ converges or maximum iteration number is achieved. We initialize $\balpha$ as ${\balpha}^{(0)} = \balpha_{\ell^{1}}$ and $\balpha_{\ell^{1}}$ is the sparse codes generated by SSC or $\ell^{1}$-graph via solving (\ref{eq:ssc-l1-lasso}) with some proper weighting parameter $\lambda_{\ell^{1}}$. In all the experimental results of this paper, we empirically set $\lambda_{\ell^{1}} = 0.1$ when initializing $\ell^{0}$-graph.

The data clustering algorithm by $\ell^{0}$-graph is described in Algorithm~\ref{alg:l0graph}. Also, the following theorem shows that each iteration of the proposed proximal method decreases the value of the objective function $L(\cdot)$ in (\ref{eq:l0graph}), therefore, our proximal method always converges.
\begin{MyTheorem}\label{theorem::sufficient-decrease}
Let $s = 2 \sigma_{\max}({\bmX^\top}{\bmX})$ where $\sigma_{\max}(\cdot)$ indicates the largest eigenvalue of a matrix, then the sequence $\{L(\balpha^{(t)})\}$ generated by the proximal method with (\ref{eq:l0graph-proximal-step1}) and (\ref{eq:l0graph-proximal-step2}) decreases, and the following inequality holds for $t \ge 1$:
\begin{small}\begin{align}\label{eq:l0graph-proximal-sufficient-decrease}
&L(\balpha^{(t)}) \le L(\balpha^{(t-1)}) - \frac{(\tau-1)s}{2} \|{\balpha}^{(t)} - {\balpha}^{(t-1)}\|_F^2
\end{align}\end{small}
And it follows that the sequence $\{L(\balpha^{(t)})\}$ converges.
\end{MyTheorem}
Furthermore, we show that if the sequence $\{\balpha^{(t)}\}$ generated by the proposed proximal method is bounded, then it is a Cauchy sequence and it converges to a critical point of the objective function $L$ in (\ref{eq:l0graph}).
\begin{MyTheorem}\label{theorem::converge-to-critical-point}
Suppose that the sequence $\{\balpha^{(t)}\}$ generated by the proximal method with (\ref{eq:l0graph-proximal-step1}) and (\ref{eq:l0graph-proximal-step2}) is bounded, then
1) $\sum\limits_{t=1}^{\infty} \|\balpha^{(t)} - \balpha^{(t-1)}\|_F < \infty$
2) $\{\balpha^{(t)}\}$ converges to a critical point \footnote{$x$ is a critical point of function $f$ if $0 \in \partial f (x)$, where $\partial f(x)$ is the limiting-subdifferential of $f$ at $x$. Please refer to more detailed definition in \cite{BoltePAL2014}. } of the function $L(\cdot)$ in (\ref{eq:l0graph}).
\end{MyTheorem}
\begin{proof}[Sketch of the Proof]
\cite{BoltePAL2014} shows that the $\ell^{0}$-norm function $\|\cdot\|_0$ is a semi-algebraic function. The conclusions of this theorem directly follows from Theorem 1 in \cite{BoltePAL2014}.
\end{proof}
The detailed proofs of Theorem~\ref{theorem::sufficient-decrease} and Theorem~\ref{theorem::converge-to-critical-point} are included in the supplementary document.

\begin{algorithm}[h]
\renewcommand{\algorithmicrequire}{\textbf{Input:}}
\renewcommand\algorithmicensure {\textbf{Output:} }
\caption{Data Clustering by $\ell^{0}$-Graph}
\label{alg:l0graph}
\begin{algorithmic}[1]
\REQUIRE ~~\\
The data set ${\bmX}=\{\bx_i\}_{i=1}^{n}$, the number of clusters $c$, the parameter $\lambda$ for $\ell^{0}$-graph, $\lambda_{\ell^{1}}$ for the initialization of the the $\ell^{0}$-graph, maximum iteration number $M$, stopping threshold $\varepsilon$\\
\STATE $t=1$, initialize the coefficient matrix as ${\balpha}^{(0)} = \balpha_{\ell^{1}}$, $s = 2 \sigma_{\max}({\bmX^\top}{\bmX})$.

\WHILE{$t \le M$}
\STATE{Obtain ${\balpha}^{(t)}$ from ${\balpha}^{(t-1)}$ by (\ref{eq:l0graph-proximal-step1}) and (\ref{eq:l0graph-proximal-step2})}
\IF{$|L(\balpha^{(t)})-L(\balpha^{(t-1)})| < \varepsilon$}
\PRINT
\ELSE
\STATE{$t=t+1$.}
\ENDIF
\ENDWHILE

\STATE{Obtain the sub-optimal coefficient matrix ${\balpha^{*}}$ when the above iterations converge or maximum iteration number is achieved.}
\STATE{Build the sparse similarity matrix by symmetrizing $\balpha^{*}$: $\bW^{*} = \frac{|\balpha^{*}|+|\balpha^{*}|^\top}{2}$, compute the corresponding normalized graph Laplacian
$\bL^{*} = (\bD^{*})^{-\frac{1}{2}}(\bD^{*}-\bW^{*})(\bD^{*})^{-\frac{1}{2}}$, where $\bD^{*}$ is a diagonal matrix with $\bD_{ii}^{*} = \sum\limits_{j=1}^n {\bW_{ij}^{*}}$}
\STATE{Construct the matrix  $\bv = [\bv_1,\ldots,\bv_c] \in \R^{n \times c}$, where $\{\bv_1,\ldots,\bv_c\}$ are the $c$ eigenvectors of $\bL^{*}$ corresponding to its $c$ smallest eigenvalues. Treat each row of $\bv$ as a data point in $\R^c$, and run K-means clustering method to obtain the cluster labels for all the rows of $\bv$. }
\ENSURE The cluster label of $\bx_i$ is set as the cluster label of the $i$-th row of $\bv$, $1 \le i \le n$.
\end{algorithmic}
\end{algorithm}

\section{Regularized $\ell^{0}$-Graph}
While the subspace-sparse representation separates the data belonging to different subspaces in the constructed sparse similarity graph, it is not guaranteed that the data points in the same subspace form a connected component. This is the well known graph connectivity issue in the sparse subspace clustering literature \cite{ElhamifarV13,Nasihatkon11} which is the only gap that prevents a sparse similarity graph with subspace-sparse representation from forming the perfect clustering result, i.e. the data belonging to each subspace form a single connected component in the sparse similarity graph. SSC \cite{ElhamifarV13} suggests alleviating the graph connectivity issue by promoting common neighbors across the data in each subspace. In this section we propose Regularized $\ell^{0}$-Graph by adding a regularization term to (\ref{eq:l0graph}) which employs $\ell^{0}$-distance between the sparse representation of the data so as to impose the sparsity of the representation and encourage common neighbors for nearby data simultaneously. Regularized $\ell^{0}$-graph solves the following problem
\begin{small}
\begin{align}\label{eq:rl0graph}
&\mathop {\min }\limits_{\balpha} \|{\bmX} - {\bmX}\balpha\|_F^2 + \gamma  R_{\bS}(\balpha)
\end{align}
\end{small}
where $R_{\bS}(\balpha) = \sum\limits_{i,j=1}^n {\bS_{ij} \|\balpha^i - \balpha^j\|_0  }$ is the regularization term, $\bS$ is the adjacency matrix of the KNN graph and $\bS_{ij} = 1$ if and only if $\bx_i$ is among the $K$ nearest neighbors of $\bx_j$ in the sense of Euclidean distance. It should be emphasized that such KNN graph is a widely used strategy to identify nearby data for graph regularization in sparse coding \cite{Zheng11,YYZRl1graphBMVC2014}. $\gamma>0$ is the weighting parameter for the regularization term. Since the co-located elements of two sparse codes $\balpha^i$ and $\balpha^j$ are not exactly the same in most cases, their $\ell^{0}$-distance $\|\balpha^i - \balpha^j\|_0$ is almost always the sum of their difference in support and the number of their co-located nonzero elements, and the support of a vector is defined to be the indices of its nonzero elements. Therefore, the regularization term $R_{\bS}(\balpha)$ encourages both sparsity and common neighbors across nearby data.

We use coordinate descent to optimize (\ref{eq:rl0graph}) with respect to $\balpha^i$ in each step of the coordinate descent, with all the other sparse codes $\{\balpha^j\}_{j \neq i}$ fixed. The optimization problem for $\balpha^i$ in each step is presented below:
\begin{small}
\begin{align}\label{eq:rl0graph-cdi}
&\mathop {\min }\limits_{\balpha^i} F(\balpha^i) =  \|{\bx_i} - {\bmX}\balpha^{i}\|_2^2 + \gamma  R_{\tilde \bS}(\balpha^i)
\end{align}
\end{small}
where $R_{\tilde \bS}(\balpha^i) = \sum\limits_{j=1}^n {\tilde \bS_{ij} \|\balpha^i - \balpha^j\|_0  }$, where $\tilde \bS = \bS + \bS^\top$.

(\ref{eq:rl0graph-cdi}) can also be optimized by the proximal method in a similar manner to $\ell^{0}$-graph. In $t$-th ($t \ge 1$) iteration of our proximal method for the problem (\ref{eq:rl0graph-cdi}), gradient descent on the squared loss term of the objective function of (\ref{eq:rl0graph-cdi}) is performed by (\ref{eq:rl0-graph-proximal-step1}):
\begin{small}\begin{align}\label{eq:rl0-graph-proximal-step1}
\tilde {\balpha^i}^{(t)} = {\balpha^i}^{(t-1)} - \frac{2}{{\tau}s} ({\bmX^\top}{\bmX}{{\balpha^i}^{(t-1)}}-{\bmX^\top}{\bx_i})
\end{align}\end{small}
where $\tau$ and $s$ are the same as that in (\ref{eq:l0graph-proximal-step1}). Then ${\balpha^i}^{(t)}$ is obtained as the solution to the following $\ell^{0}$ regularized problem:
\begin{align}\label{eq:rl0graph-cdi-subprob}
&{\balpha^i}^{(t)} = \nonumber \\
&{\argmin}_{{\bv \in \R^{n}, \bv_{i} = 0}} {\frac{{\tau} s}{2}\|\bv - \tilde {\balpha^i}^{(t)}\|_2^2 + \gamma R_{\tilde \bS}(\bv)  }
\end{align}
Proposition~\ref{proposition::sol-to-subprob} below shows the closed form solution to the subproblem (\ref{eq:rl0graph-cdi-subprob}):
\begin{MyProposition}\label{proposition::sol-to-subprob}
Define $F_k(v) = \frac{{\tau} s}{2}\|v - \tilde {\balpha}_{ki}^{(t)}\|_2^2 + \gamma R_{\tilde \bS}(v)$ for $v \in \R$ and $R_{\tilde \bS}(v) \triangleq \sum\limits_{j=1}^n {\tilde \bS_{ij} \|v - \balpha_{kj}\|_0  }$. Let $\bv^{*}$ be the optimal solution to (\ref{eq:rl0graph-cdi-subprob}), then the $k$-th element of $\bv^{*}$ is
\begin{small}\begin{align}\label{eq:rl0graph-proximal-step2}
&\bv_k^{*} =
\left\{
\begin{array}
    {r@{\quad:\quad}l}
    \argmin_{v \in \{\tilde {\balpha}_{ki}^{(t)}\} \cup \{\balpha_{kj}\}_{\{j: \tilde \bS_{ij} \neq 0\}} } F_k(v) & { k \neq i} \\
    0 & k = i
\end{array}
\right.
\end{align}\end{small}
\end{MyProposition}
Proposition~\ref{proposition::sol-to-subprob} suggests an efficient way of obtaining the solution to (\ref{eq:rl0graph-cdi-subprob}). According to (\ref{eq:rl0graph-proximal-step2}), ${\balpha^i}^{(t)} = \bv^{*}$ can be obtained by searching over a candidate set of size $K+1$, where $K$ is the number of nearest neighbors to construct the KNN graph $\bS$ for regularized $\ell^{0}$-graph.

Similar to Theorem~\ref{theorem::sufficient-decrease}, the sequence $\{F({\balpha^i}^{(t)})\}_t$ is decreasing. The iterative proximal method starts from $t=1$ and continue until the sequence $\{F({\balpha^i}^{(t)})\}_t$ converges or maximum iteration number is achieved. When the proximal method converges or terminates for each $\balpha^i$, the step of coordinate descent for $\balpha^i$ is finished and the optimization algorithm proceeds to optimize other sparse codes. Each iteration of coordinate descent solves (\ref{eq:rl0graph-cdi}) for $i=1 \ldots n$ sequentially, and it terminates when maximum iteration number is reached or converges under some stopping threshold on the change of the objective function (\ref{eq:rl0graph}).

%
%
%
%
%

\begin{table*}[ht]
\centering
\caption{\small Clustering Results on Ionosphere and MNIST Handwritten Digits Database}
\begin{tabular}{|c|c|c|c|c|c|c|c|}
  \hline
  Data Set

                              &Measure & KM    & SC     &$\ell^{1}$-Graph   &SMCE    &OMP-Graph   &$\ell^{0}$-Graph   \\\hline

  \multirow{2}{*}{Ionosphere} &AC      &0.7097 &0.7350  &0.5128             &0.6809  &0.6353      &\textbf{0.7692} \\ \cline{2-8}
                              &NMI     &0.1287 &0.2155  &0.1165             &0.0871  &0.0299      &\textbf{0.2609} \\ \hline

  \multirow{2}{*}{MNIST}      &AC      &0.5621 &0.4922  &0.4948             &0.5784  &0.5754      &\textbf{0.6590} \\ \cline{2-8}
                              &NMI     &0.5113 &0.4755  &0.5210             &0.6332  &0.5463      &\textbf{0.6709} \\ \hline

\end{tabular}
\label{table:uci-mnist}
\end{table*}

\begin{table*}[ht]
\centering
\small
\caption{\small Clustering Results on COIL-20 Database. $c$ in the left column is the cluster number, i.e. the first $c$ clusters of the entire data are used for clustering. $c$ has the same meaning in Table~\ref{table:coil100} and Table ~\ref{table:yaleb}.}
\begin{tabular}{|c|c|c|c|c|c|c|c|}
  \hline
  \begin{tabular}{c}COIL-20 \\ \hline
  \# Clusters \end{tabular}

                           &Measure & KM    & SC     &$\ell^{1}$-Graph   &SMCE    &OMP-Graph   &$\ell^{0}$-Graph   \\\hline

  \multirow{2}{*}{c = 4}   &AC      &0.6632 &0.6701  &1.0000             &0.7639  &0.9271      &\textbf{1.0000} \\ \cline{2-8}
                           &NMI     &0.5106 &0.5455  &1.0000             &0.6741  &0.8397      &\textbf{1.0000} \\ \hline

  \multirow{2}{*}{c = 8}   &AC      &0.5130 &0.4462  &0.7986             &0.5365  &0.6753      &\textbf{0.9705} \\ \cline{2-8}
                           &NMI     &0.5354 &0.4947  &0.8950             &0.6786  &0.7656      &\textbf{0.9638} \\ \hline

  \multirow{2}{*}{c = 12}  &AC      &0.5885 &0.4965  &0.7697             &0.6806  &0.5475      &\textbf{0.8310} \\ \cline{2-8}
                           &NMI     &0.6707 &0.6096  &0.8960             &0.8066  &0.6316      &\textbf{0.9149} \\ \hline

  \multirow{2}{*}{c = 16}  &AC      &0.6579 &0.4271  &0.8273             &0.7622  &0.3481      &\textbf{0.9002} \\ \cline{2-8}
                           &NMI     &0.7555 &0.6031  &0.9301             &0.8730  &0.4520      &\textbf{0.9552} \\ \hline

  \multirow{2}{*}{c = 20}  &AC      &0.6554 &0.4278  &0.7854             &0.7549  &0.3389      &\textbf{0.8472} \\ \cline{2-8}
                           &NMI     &0.7630 &0.6217  &0.9148             &0.8754  &0.4853      &\textbf{0.9428} \\ \hline

\end{tabular}
\label{table:coil20}
\end{table*}
\begin{table*}[ht]
\centering
\small
\caption{\small Clustering Results on COIL-100 Database.}
\begin{tabular}{|c|c|c|c|c|c|c|c|}
  \hline
  \begin{tabular}{c}COIL-100 \\ \hline
  \# Clusters \end{tabular}

                           &Measure & KM    & SC     &$\ell^{1}$-Graph   &SMCE    &OMP-Graph   &$\ell^{0}$-Graph   \\\hline

  \multirow{2}{*}{c = 20}  &AC      &0.5850 &0.4514  &0.5757             &0.6208  &0.4243      &\textbf{0.9264} \\ \cline{2-8}
                           &NMI     &0.7456 &0.6700  &0.7980             &0.7993  &0.5258      &\textbf{0.9681} \\ \hline

  \multirow{2}{*}{c = 40}  &AC      &0.5791 &0.4139  &0.5934             &0.6038  &0.2340      &\textbf{0.8472} \\ \cline{2-8}
                           &NMI     &0.7691 &0.6681  &0.7962             &0.7918  &0.4378      &\textbf{0.9471} \\ \hline

  \multirow{2}{*}{c = 60}  &AC      &0.5371 &0.3389  &0.5657             &0.5887  &0.1905      &\textbf{0.8326} \\ \cline{2-8}
                           &NMI     &0.7622 &0.6343  &0.8162             &0.7973  &0.3690      &\textbf{0.9352} \\ \hline

  \multirow{2}{*}{c = 80}  &AC      &0.5048 &0.3115  &0.5271             &0.5835  &0.2247      &\textbf{0.7899} \\ \cline{2-8}
                           &NMI     &0.7474 &0.6088  &0.8006             &0.8006  &0.4173      &\textbf{0.9218} \\ \hline

  \multirow{2}{*}{c = 100} &AC      &0.4996 &0.2835  &0.5275             &0.5639  &0.1667      &\textbf{0.7683} \\ \cline{2-8}
                           &NMI     &0.7539 &0.5923  &0.8041             &0.8064  &0.3757      &\textbf{0.9182} \\ \hline

\end{tabular}
\label{table:coil100}
\end{table*}
\begin{table*}[ht]
\centering
\small
\caption{\small Clustering Results on the Extended Yale Face Database B.}
\begin{tabular}{|c|c|c|c|c|c|c|c|}
  \hline
  \begin{tabular}{c}Yale-B \\ \hline
  \# Clusters \end{tabular}

                           &Measure & KM    & SC     &$\ell^{1}$-Graph   &SMCE    &OMP-Graph   &$\ell^{0}$-Graph   \\\hline

  \multirow{2}{*}{c = 10}  &AC      &0.1782 &0.1922  &0.7580             &0.3672  &0.7375      &\textbf{0.8406} \\ \cline{2-8}
                           &NMI     &0.0897 &0.1310  &0.7380             &0.3266  &0.7468      &\textbf{0.7695} \\ \hline

  \multirow{2}{*}{c = 15}  &AC      &0.1554 &0.1706  &0.7620             &0.3761  &0.7532      &\textbf{0.7987} \\ \cline{2-8}
                           &NMI     &0.1083 &0.1390  &0.7590             &0.3593  &0.7943      &\textbf{0.8183} \\ \hline

  \multirow{2}{*}{c = 20}  &AC      &0.1200 &0.1466  &0.7930             &0.3526  &0.7813      &\textbf{0.8273} \\ \cline{2-8}
                           &NMI     &0.0872 &0.1183  &0.7860             &0.3771  &0.8172      &\textbf{0.8429} \\ \hline

  \multirow{2}{*}{c = 30}  &AC      &0.1096 &0.1209  &0.8210             &0.3470  &0.7156      &\textbf{0.8633} \\ \cline{2-8}
                           &NMI     &0.1159 &0.1338  &0.8030             &0.3927  &0.7260      &\textbf{0.8762} \\ \hline

  \multirow{2}{*}{c = 38}  &AC      &0.0954 &0.1077  &0.7850             &0.3293  &0.6529      &\textbf{0.8480} \\ \cline{2-8}
                           &NMI     &0.1258 &0.1485  &0.7760             &0.3812  &0.7024      &\textbf{0.8612} \\ \hline

\end{tabular}
\label{table:yaleb}
\end{table*}

\begin{figure*}[!htb]
\begin{center}
\includegraphics[width=0.46\textwidth]{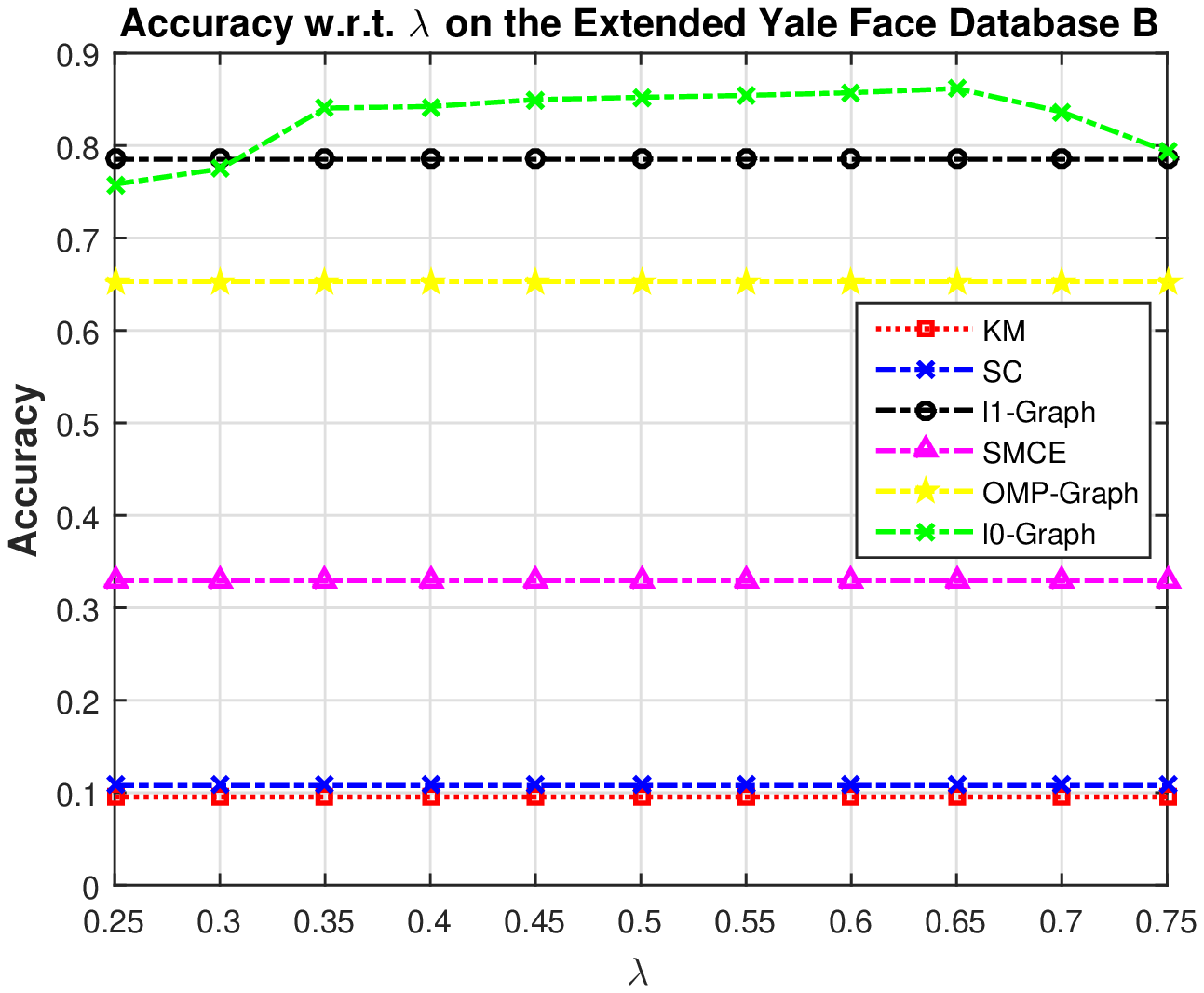}
\includegraphics[width=0.46\textwidth]{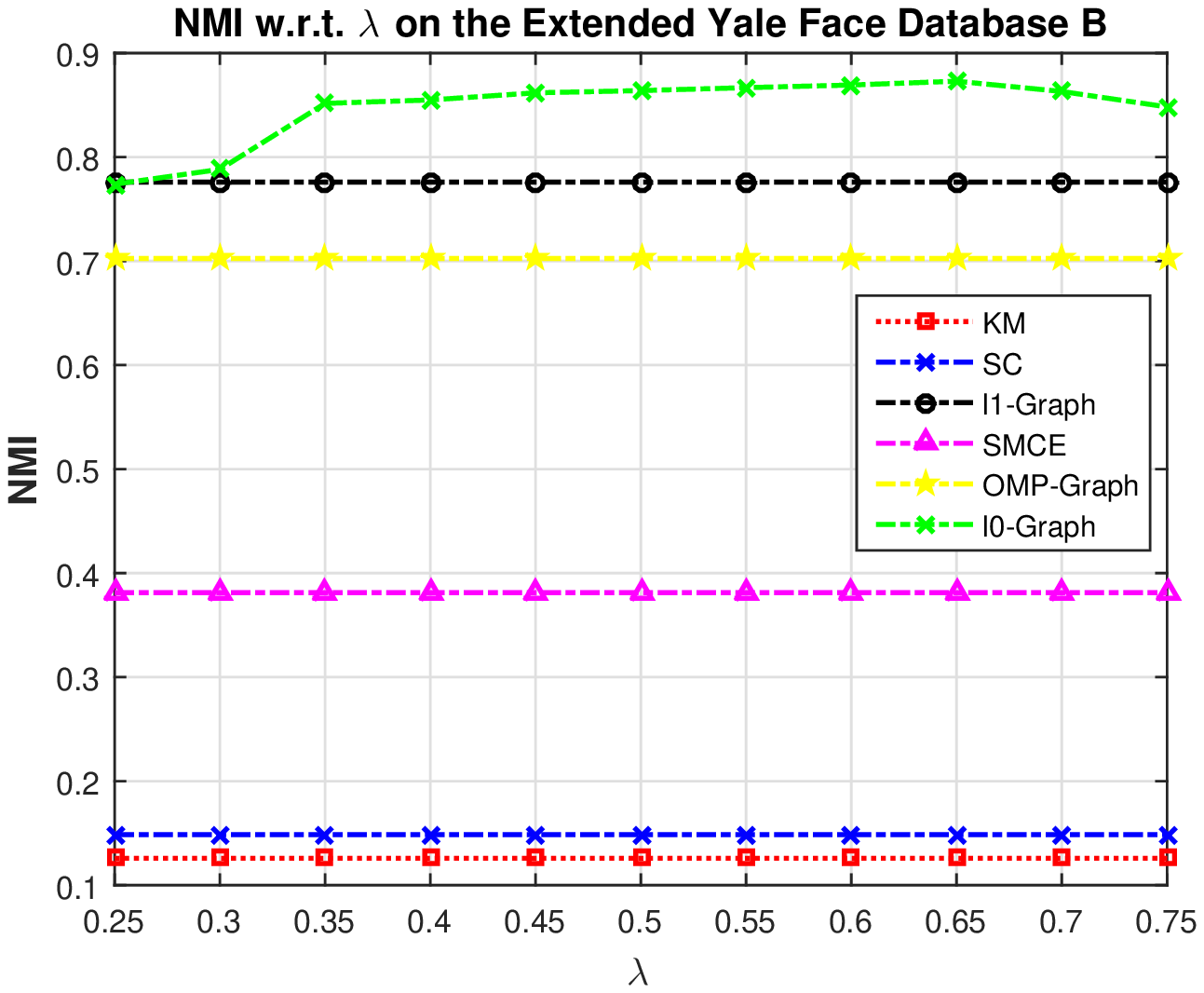}
\end{center}
   \caption{Clustering performance with different values of $\lambda$, i.e. the weight for the $\ell^{0}$-norm, on the Extended Yale Face Database B. Left: Accuracy; Right: NMI}
\label{fig:yaleb-lambda}
\end{figure*}

\begin{figure*}[!htb]
\begin{center}
\includegraphics[width=0.46\textwidth]{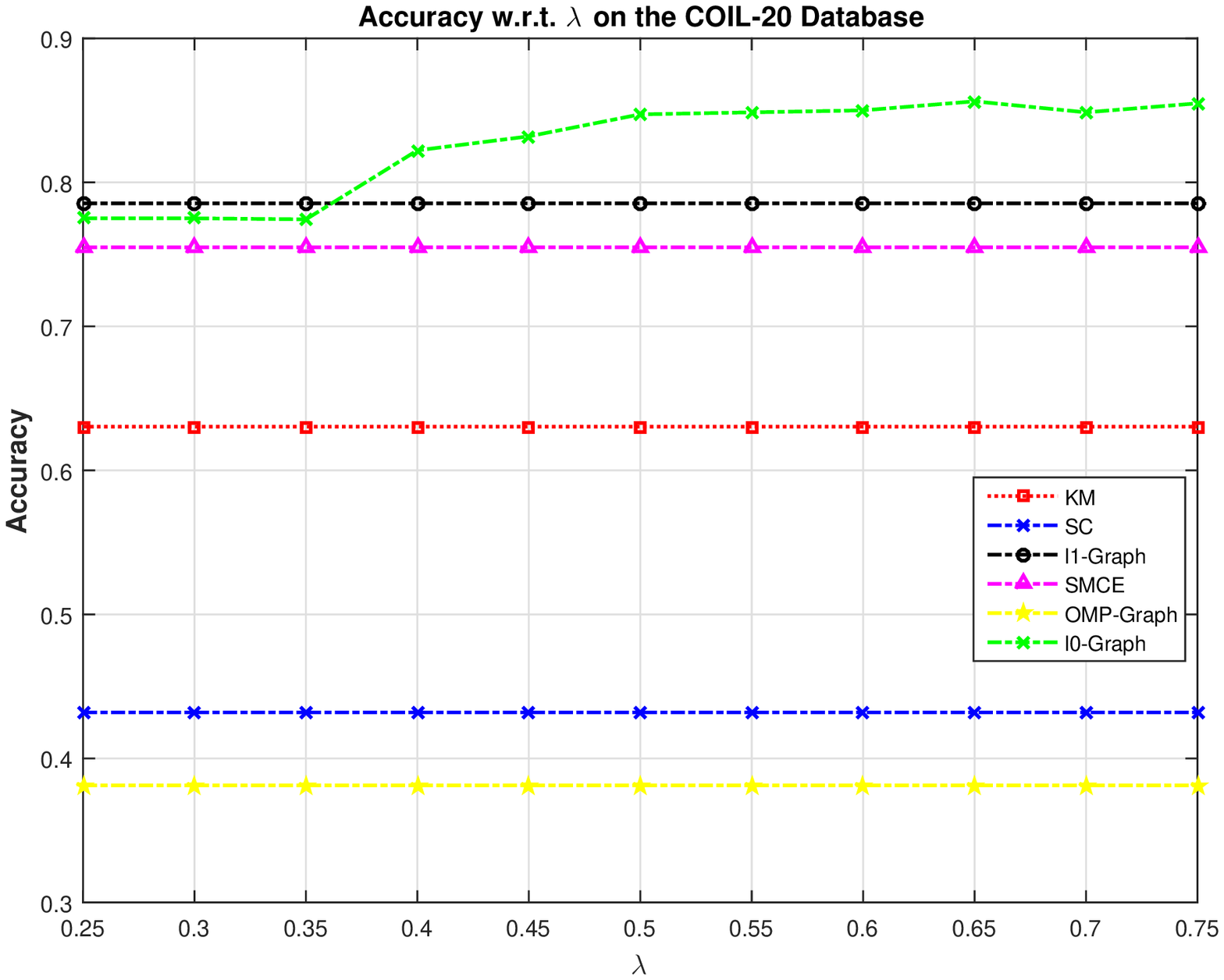}
\includegraphics[width=0.46\textwidth]{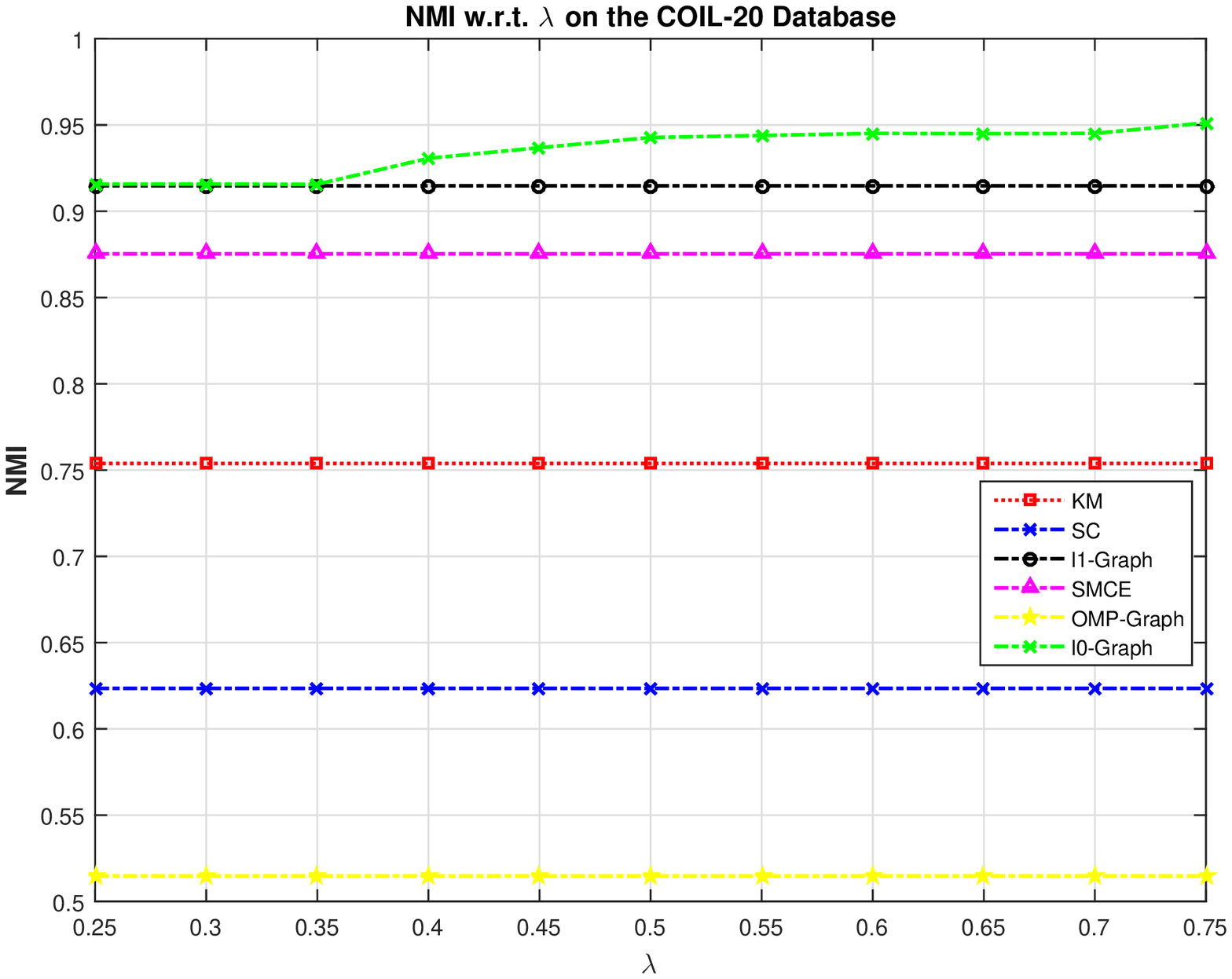}
\end{center}
   \caption{Clustering performance with different values of $\lambda$, i.e. the weight for the $\ell^{0}$-norm, on the COIL-20 Database. Left: Accuracy; Right: NMI}
\label{fig:coil20-lambda}
\end{figure*}

\section{Experimental Results}
The superior clustering performance of $\ell^{0}$-graph is demonstrated in this section with extensive experimental results, and we also show the effectiveness of regularized $\ell^{0}$-graph.  We compare our $\ell^{0}$-graph to K-means (KM), Spectral Clustering (SC), $\ell^{1}$-graph, Sparse Manifold Clustering and Embedding (SMCE) \cite{ElhamifarV11}. Moreover, we derive the OMP-graph, which builds the sparse graph in the same way as $\ell^{0}$-graph except that it solves the following optimization problem by Orthogonal Matching Pursuit (OMP) to obtain the sparse code:
\begin{small}\begin{align}\label{eq:ompgraph}
\mathop {\min }\limits_{{\balpha^i}} \|\bx_i - \bmX \balpha^{i}\|_F^2  \quad s.t. \,\, \|{\balpha^{i}}\|_0 \le T, {\balpha}_i^i = 0, \,\, i=1,\ldots,n
\end{align}\end{small}
$\ell^{0}$-graph is also compared to OMP-graph to show the advantage of the proposed proximal method in the previous sections. By adjusting the parameters, $\ell^{1}$-graph and SSC solve the same problem and generate equivalent results, so we report their performance under the same name ``$\ell^{1}$-graph''.
\subsection{Evaluation Metric}
Two measures are used to evaluate the performance of the clustering
methods, i.e. the accuracy and the Normalized Mutual Information(NMI) \cite{Zheng04}. Let the predicted label of the datum $\bx_i$ be $\hat y_i$ which is produced by the clustering method, and $y_i$ is its ground truth label. The accuracy is defined as
\begin{align}\label{eq:accuracy}
&{Accuracy} = \frac{\1_{{\Omega}(\hat y_i) \ne y_i}}{n}
\end{align}
\noindent where $\1$ is the indicator function, and $\Omega$ is the best permutation mapping
function by the Kuhn-Munkres algorithm \cite{plummer1986}. The more predicted labels match the ground truth ones, the more accuracy value is obtained.

Let $\hat X$ be the index set obtained from the predicted labels $\{\hat y_i\}_{i=1}^n$ and $X$ be the index set from the ground truth labels $\{y_i\}_{i=1}^n$. The mutual information between ${\hat X}$ and $X$ is
\begin{small}\begin{align}\label{eq:MI}
MI( {\hat X,X}) = \sum\limits_{\hat x \in \hat X,x \in X} {p( {\hat x,x} ){{\log }_2}( {\frac{{p( {\hat x,x} )}}{{p( {\hat x} )p( x )}}} )}
\end{align}\end{small}
\noindent where $p(\hat x)$ and $p(x)$ are the margined distribution of $\hat X$ and $X$ respectively, induced from the joint distribution $p(\hat x, x)$ over $\hat X$ and $X$. Let $H( {\hat X} )$ and $H( X )$ be the entropy of $\hat X$ and $X$, then the normalized mutual information (NMI) is defined as below:
\begin{small}\begin{align}\label{eq:NMI}
&NMI( {\hat X,X} ) = \frac{{MI( {\hat X,X} )}}{{\max \{ {H( {\hat X} ),H( X )}\}}}
\end{align}\end{small}
It can be verified that the normalized mutual information takes values in $[0,1]$. The accuracy and the normalized mutual information have been widely used for evaluating the performance of the clustering methods \cite{Zheng11,ChengYYFH10,Zheng04}.

\subsection{Clustering on UCI Data Set and MNIST Handwritten Digits Database}
In this subsection, we conduct experiments on the Ionosphere data from UCI machine learning repository \cite{Asuncion07} and the MNIST database of handwritten digits. The information of these two data sets are in Table~\ref{table:datasets}. MNIST handwritten digits database has a total number of $70000$ samples for digits from $0$ to $9$. The digits are normalized and centered in a fixed-size image. For MNIST data set, we randomly select $500$ samples for each digit to obtain a subset of MNIST data consisting of $5000$ samples. The random sampling is performed for $10$ times and the average clustering performance is recorded. The clustering results on the two data sets are shown in Table~\ref{table:uci-mnist}.
\begin{table}[h]
\centering
\small
\caption{Two UCI data sets and MNIST Handwritten Digits Database in the experiments}
\begin{tabular}{|c|c|c|c|c|c|c|c|c|c|c|}
  \hline
                       &Heart &Ionosphere  &MNIST                \\\hline
  \# of instances      &270  &351          &70000                         \\ \hline
  Dimension            &13   &34           &1024                         \\ \hline
  \# of classes        &2    &2            &10                    \\ \hline
\end{tabular}
\label{table:datasets}
\end{table}

\subsection{Clustering On COIL-20 and COIL-100 Database}
COIL-20 Database has $1440$ images of $20$ objects in which the background has been removed, and the size of each image is $32 \times 32$, so the dimension of this data is $1024$. COIL-100 Database contains $100$ objects with $72$ images of size $32 \times 32$ for each object. The images of each object were taken $5$ degrees apart when the object was rotated on a turntable. The clustering results on these two data sets are shown in Table~\ref{table:coil20} and Table~\ref{table:coil100} respectively. We observe that $\ell^{0}$-graph performs consistently better than all other competing methods. On COIL-100 Database, SMCE renders slightly better results than $\ell^{1}$-graph on the entire data due to its capability of modeling non-linear manifolds.

\subsection{Clustering On Extended Yale Face Database B}
The Extended Yale Face Database B contains face images for $38$ subjects with $64$ frontal face images taken under different illuminations for each subject. The clustering results are shown in Table~\ref{table:yaleb}. We can see that $\ell^{0}$-graph achieves significantly better clustering result than $\ell^{1}$-graph, which is the second best method on this data.

\subsection{Improved $\ell^{0}$-Graph with Regularization}
In this subsection, we investigate the performance of regularized $\ell^{0}$-graph. We empirically set $\bS$ to be the the adjacency matrix of $5$-NN graph and $\gamma = 0.1$ as the default parameter setting for regularized $\ell^{0}$-graph in (\ref{eq:rl0graph}). We conduct comparison experiments on the UCI Heart data whose information is in Table~\ref{table:uci-mnist}, the Extended Yale Face Database B and UMIST Face Database. The UMIST Face Database consists of $575$ images of size $112 \times 92$ for $20$ people. Each person is shown in a range of poses from profile to frontal views. The clustering results are shown in Table~\ref{table:rl0graph}. The better results of regularized $\ell^{0}$-graph are due to the fact that it promotes common neighbors for nearby data so as to produce a more aligned similarity graph and alleviate the graph connectivity issue.
\begin{table}[ht]
\centering
\caption{\small Clustering Performance of Regularized $\ell^{0}$-Graph}
\begin{tabular}{|c|c|c|c|c|c|c|c|}
  \hline
  Data Set

                                        &Measure &$\ell^{0}$-Graph &R$\ell^{0}$-Graph   \\\hline

  \multirow{2}{*}{Heart}                &AC      &0.5111 &\textbf{0.6444}  \\ \cline{2-4}
                                        &NMI     &0.0064 &\textbf{0.0590}  \\ \hline

  \multirow{2}{*}{Extended Yale B}      &AC      &0.8480 &\textbf{0.8521}  \\ \cline{2-4}
                                        &NMI     &0.8612 &\textbf{0.8634} \\ \hline

  \multirow{2}{*}{UMIST Face}           &AC      &0.6730 &\textbf{0.7078}  \\ \cline{2-4}
                                        &NMI     &0.7924 &\textbf{0.8153} \\ \hline
\end{tabular}
\label{table:rl0graph}
\end{table}

\subsection{Parameter Setting}
We use the sparse codes generated by $\ell^{1}$-graph with the weighting parameter $\lambda_{\ell^{1}} = 0.1$ in (\ref{eq:ssc-l1-lasso}), which is the default value suggested in \cite{ElhamifarV13}, to initialize $\ell^{0}$-graph, and set $\lambda=0.5$ for $\ell^{0}$-graph empirically throughout all the experiments in this section. We observe that the average number of non-zero elements of the sparse code for each data point is around $3$ for most data sets. The maximum iteration number $M = 100$ and the stopping threshold $\varepsilon = 10^{-6}$. For OMP-graph, we tune the parameter $T$ in (\ref{eq:ompgraph}) to control the sparsity of the generated sparse codes such that the aforementioned average number of non-zero elements of the sparse code matches that of $\ell^{0}$-graph. For $\ell^{1}$-graph, the weighting parameter for the $\ell^{1}$-norm is chosen from $[0.1,1]$ for the best performance.

We investigate how the clustering performance on the Extended Yale Face Database B and COIL-20 Database changes by varying the weighting parameter $\lambda$ for $\ell^{0}$-graph, and illustrate the result in Figure~\ref{fig:yaleb-lambda} and Figure~\ref{fig:coil20-lambda} respectively. We observe that the performance of $\ell^{0}$-graph is much better than other algorithms over a relatively large range of $\lambda$, revealing the robustness of our algorithm with respect to the weighting parameter $\lambda$.

\subsection{Efficient Parallel Computing by CUDA Implementation}
We have implemented $\ell^{0}$-graph, regularized $\ell^{0}$-graph in CUDA C programming language on NVIDIA K$40$. Both the MATLAB and CUDA implementation will be available for downloading. We compare the running time of $\ell^{0}$-graph in MATLAB implementation and CUDA C implementation on the Extended Yale Face Database B data, on a workstation with $2$ Intel Xeon X$5650$ $2.67$ GHz CPU, $48$ GB memory and one NVIDIA K$40$ graphics card. MATLAB implementation takes $48.51$ seconds while the CUDA implementation only takes $1.68$ seconds, with a speedup of $28.87$ times.

Due to the limited space, we have put additional experimental results in the supplementary document for this paper, such as the application of $\ell^{0}$-graph on semi-supervised learning, and the parameter sensitivity for regularized $\ell^{0}$-graph.

\section{Conclusion}
We propose a novel $\ell^{0}$-graph for data clustering in this paper. In contrast to the existing sparse subspace clustering method such as Sparse Subspace Clustering and $\ell^{1}$-graph, $\ell^{0}$-graph features $\ell^{0}$-induced almost surely subspace-sparse representation under milder assumptions on the subspaces and random data generation. The objective function of $\ell^{0}$-graph is optimized using a proposed proximal method. Convergence of this proximal method is proved, and extensive experimental results on various real data sets demonstrate the effectiveness and superiority of $\ell^{0}$-graph over other competing methods. To improve the graph connectivity, we propose regularized $\ell^{0}$-graph whose effectiveness is also demonstrated on real data sets.


{\small
\bibliographystyle{ieee}
\bibliography{egbib}

\begin{thebibliography}{10}\itemsep=-1pt

\bibitem{Asuncion07}
D.~N. A.~Asuncion.
\newblock {UCI} machine learning repository, 2007.

\bibitem{BaoJQS14}
C.~Bao, H.~Ji, Y.~Quan, and Z.~Shen.
\newblock {L0} norm based dictionary learning by proximal methods with global
  convergence.
\newblock In {\em 2014 {IEEE} Conference on Computer Vision and Pattern
  Recognition, {CVPR} 2014, Columbus, OH, USA, June 23-28, 2014}, pages
  3858--3865, 2014.

\bibitem{BoltePAL2014}
J.~Bolte, S.~Sabach, and M.~Teboulle.
\newblock Proximal alternating linearized minimization for nonconvex and
  nonsmooth problems.
\newblock {\em Math. Program.}, 146(1-2):459--494, Aug. 2014.

\bibitem{ChengYYFH10}
B.~Cheng, J.~Yang, S.~Yan, Y.~Fu, and T.~S. Huang.
\newblock Learning with l1-graph for image analysis.
\newblock {\em IEEE Transactions on Image Processing}, 19(4):858--866, 2010.

\bibitem{ChengSRL2013}
H.~Cheng, Z.~Liu, L.~Yang, and X.~Chen.
\newblock Sparse representation and learning in visual recognition: Theory and
  applications.
\newblock {\em Signal Process.}, 93(6):1408--1425, June 2013.

\bibitem{Dyer13a}
E.~L. Dyer, A.~C. Sankaranarayanan, and R.~G. Baraniuk.
\newblock Greedy feature selection for subspace clustering.
\newblock {\em Journal of Machine Learning Research}, 14:2487--2517, 2013.

\bibitem{ElhamifarV11}
E.~Elhamifar and R.~Vidal.
\newblock Sparse manifold clustering and embedding.
\newblock In {\em NIPS}, pages 55--63, 2011.

\bibitem{ElhamifarV13}
E.~Elhamifar and R.~Vidal.
\newblock Sparse subspace clustering: Algorithm, theory, and applications.
\newblock {\em {IEEE} Trans. Pattern Anal. Mach. Intell.}, 35(11):2765--2781,
  2013.

\bibitem{Fraley02}
C.~Fraley and A.~E. Raftery.
\newblock {Model-Based Clustering, Discriminant Analysis, and Density
  Estimation}.
\newblock {\em Journal of the American Statistical Association},
  97(458):611--631, June 2002.

\bibitem{Hyder09}
M.~Hyder and K.~Mahata.
\newblock An approximate l0 norm minimization algorithm for compressed sensing.
\newblock In {\em Acoustics, Speech and Signal Processing, 2009. ICASSP 2009.
  IEEE International Conference on}, pages 3365--3368, April 2009.

\bibitem{jenatton2010proximal}
R.~Jenatton, J.~Mairal, F.~R. Bach, and G.~R. Obozinski.
\newblock Proximal methods for sparse hierarchical dictionary learning.
\newblock In {\em Proceedings of the 27th International Conference on Machine
  Learning (ICML-10)}, pages 487--494, 2010.

\bibitem{Liu12}
G.~Liu, Z.~Lin, S.~Yan, J.~Sun, Y.~Yu, and Y.~Ma.
\newblock Robust recovery of subspace structures by low-rank representation.
\newblock {\em IEEE Trans. Pattern Anal. Mach. Intell.}, 35(1):171--184, Jan.
  2013.

\bibitem{LiuLY10}
G.~Liu, Z.~Lin, and Y.~Yu.
\newblock Robust subspace segmentation by low-rank representation.
\newblock In {\em Proceedings of the 27th International Conference on Machine
  Learning (ICML-10), June 21-24, 2010, Haifa, Israel}, pages 663--670, 2010.

\bibitem{Mairal2010}
J.~Mairal, F.~Bach, J.~Ponce, and G.~Sapiro.
\newblock Online learning for matrix factorization and sparse coding.
\newblock {\em J. Mach. Learn. Res.}, 11:19--60, Mar. 2010.

\bibitem{MairalBPSZ08}
J.~Mairal, F.~R. Bach, J.~Ponce, G.~Sapiro, and A.~Zisserman.
\newblock Supervised dictionary learning.
\newblock In {\em Advances in Neural Information Processing Systems 21,
  Proceedings of the Twenty-Second Annual Conference on Neural Information
  Processing Systems, Vancouver, British Columbia, Canada, December 8-11,
  2008}, pages 1033--1040, 2008.

\bibitem{Mancera2006}
L.~Mancera and J.~Portilla.
\newblock L0-norm-based sparse representation through alternate projections.
\newblock In {\em Image Processing, 2006 IEEE International Conference on},
  pages 2089--2092, Oct 2006.

\bibitem{Nasihatkon11}
B.~Nasihatkon and R.~Hartley.
\newblock Graph connectivity in sparse subspace clustering.
\newblock In {\em Computer Vision and Pattern Recognition (CVPR), 2011 IEEE
  Conference on}, pages 2137--2144, June 2011.

\bibitem{Ng01}
A.~Y. Ng, M.~I. Jordan, and Y.~Weiss.
\newblock On spectral clustering: Analysis and an algorithm.
\newblock In {\em NIPS}, pages 849--856, 2001.

\bibitem{ParkCS14}
D.~Park, C.~Caramanis, and S.~Sanghavi.
\newblock Greedy subspace clustering.
\newblock In {\em Advances in Neural Information Processing Systems 27: Annual
  Conference on Neural Information Processing Systems 2014, December 8-13 2014,
  Montreal, Quebec, Canada}, pages 2753--2761, 2014.

\bibitem{Peng2015robust}
X.~Peng, Z.~Yi, and H.~Tang.
\newblock Robust subspace clustering via thresholding ridge regression.
\newblock In {\em AAAI Conference on Artificial Intelligence (AAAI)}, pages
  3827--3833. AAAI, 2015.

\bibitem{plummer1986}
D.~Plummer and L.~Lov{\'a}sz.
\newblock {\em Matching Theory}.
\newblock North-Holland Mathematics Studies. Elsevier Science, 1986.

\bibitem{Soltanolkotabi2012}
M.~Soltanolkotabi and E.~J. Candés.
\newblock A geometric analysis of subspace clustering with outliers.
\newblock {\em Ann. Statist.}, 40(4):2195--2238, 08 2012.

\bibitem{Tropp04}
J.~A. Tropp.
\newblock Greed is good: algorithmic results for sparse approximation.
\newblock {\em {IEEE} Transactions on Information Theory}, 50(10):2231--2242,
  2004.

\bibitem{Vidal11}
R.~Vidal.
\newblock Subspace clustering.
\newblock {\em Signal Processing Magazine, IEEE}, 28(2):52--68, March 2011.

\bibitem{Wang13}
Y.-X. Wang, H.~Xu, and C.~Leng.
\newblock Provable subspace clustering: When lrr meets ssc.
\newblock In C.~Burges, L.~Bottou, M.~Welling, Z.~Ghahramani, and
  K.~Weinberger, editors, {\em Advances in Neural Information Processing
  Systems 26}, pages 64--72. Curran Associates, Inc., 2013.

\bibitem{YanW09}
S.~Yan and H.~Wang.
\newblock Semi-supervised learning by sparse representation.
\newblock In {\em SDM}, pages 792--801, 2009.

\bibitem{YangYGH09}
J.~Yang, K.~Yu, Y.~Gong, and T.~S. Huang.
\newblock Linear spatial pyramid matching using sparse coding for image
  classification.
\newblock In {\em CVPR}, pages 1794--1801, 2009.

\bibitem{YYZRl1graphBMVC2014}
Y.~Yang, Z.~Wang, J.~Yang, J.~Han, and T.~Huang.
\newblock Regularized l1-graph for data clustering.
\newblock In {\em Proceedings of the British Machine Vision Conference}. BMVA
  Press, 2014.

\bibitem{ZhangGLXA13}
T.~Zhang, B.~Ghanem, S.~Liu, C.~Xu, and N.~Ahuja.
\newblock Low-rank sparse coding for image classification.
\newblock In {\em {IEEE} International Conference on Computer Vision, {ICCV}
  2013, Sydney, Australia, December 1-8, 2013}, pages 281--288, 2013.

\bibitem{Zheng11}
M.~Zheng, J.~Bu, C.~Chen, C.~Wang, L.~Zhang, G.~Qiu, and D.~Cai.
\newblock Graph regularized sparse coding for image representation.
\newblock {\em IEEE Transactions on Image Processing}, 20(5):1327--1336, 2011.

\bibitem{Zheng04}
X.~Zheng, D.~Cai, X.~He, W.-Y. Ma, and X.~Lin.
\newblock Locality preserving clustering for image database.
\newblock In {\em Proceedings of the 12th Annual ACM International Conference
  on Multimedia}, MULTIMEDIA '04, pages 885--891, New York, NY, USA, 2004. ACM.

\end{thebibliography}
}

\end{document}